\documentclass[11pt,twocolumn]{article}

\usepackage{microtype}
\usepackage{graphicx}
\usepackage{subfigure}
\usepackage{booktabs} 
\usepackage{ntheorem}
\usepackage{hyperref}
\usepackage{amsmath}
\usepackage{amsfonts}

\RequirePackage{algorithm}
\RequirePackage{algorithmic}
\newtheorem*{proof}{Proof}
\newtheorem{prop}{Proposition}

\title{Robust spectral clustering using LASSO regularization}

\author{{\normalsize Champion Camille$^{1}$, Blazère Mélanie$^1$, Burcelin Rémy$^2$, Loubes Jean-Michel$^1$, Risser Laurent$^3$}
\\
{\normalsize $^1$ Toulouse Mathematics Institute (UMR 5219)}
\\
{\normalsize  University of Toulouse \ F-31062 Toulouse, France}
\\
{\normalsize $^2$ Metabolic and Cardiovascular Diseases Institute (UMR 1048)}
\\
{\normalsize University of Toulouse \ F-31432 Toulouse, France}
\\
{\normalsize $^3$ Toulouse Mathematics Institute (UMR 5219)}
\\
{\normalsize  CNRS \ F-31062 Toulouse, France}
}

\date{}

\begin{document}
\maketitle

\begin{abstract}
Cluster structure detection is a fundamental task for the analysis of graphs, in order to understand and to visualize their functional characteristics. Among the different cluster structure detection methods, spectral clustering is currently one of the most widely used due to its speed and simplicity. Yet, there are few theoretical guarantee to recover the underlying partitions of the graph for general models. This paper therefore presents a variant of spectral clustering, called $\ell_1$-spectral clustering, performed on a new random model closely related to stochastic block model. Its goal is to   promote a sparse eigenbasis solution of a $\ell_1$ minimization problem revealing the natural structure of the graph. The effectiveness and the robustness to small noise perturbations of our technique is confirmed through a collection of simulated and real data examples.
\end{abstract}

\vspace{0.4cm}

\vfill




\section{Introduction}\label{section1}

Graphs play a central role in complex systems  as they can conveniently model  interactions between the variables of a system.  Finding variable sets with similar attributes can then help  understanding the mechanisms underlying a complex system.
Graphs  are commonly used in  a wide range of applications, ranging from Mathematics (graph theory) to Physics \cite{Hopfield82}, Social Networks \cite{Handcock10}, Informatics  \cite{Pastor07} or Biology  \cite{Jeong00,Meunier09}. For instance, in genetics, groups of genes with high interactions are likely to be involved in a same function that drives a specific biological process.

One of the most relevant features when analyzing graphs is cluster structures.  Clusters are generally  defined as  connected subsets of nodes that are more densely connected to each other than to the rest of the graph. 
Different strategies make it possible to define more specifically variable clusters  depending on whether this property of vertices is considered locally (on a connected subset of vertices) or globally (on the whole network). 
First, cliques (subset of vertices such that every two distinct vertices in the clique are adjacent)\cite{Wasserman94}, n-clique (maximal subgraph such that the distance of each pair of its vertices is not larger than n) \cite{Wasserman94} and  k-core (maximal connected subgraph of G in which all vertices have degree at least k) \cite{Seidman83} characterize   local cluster structure.
Secondly, one of the global cluster structure definition is based on the notion of  modularity \cite{Newman04, Newman06} that quantifies the extent to which   the fraction of the edges that fall within the given groups differs from the expected fraction if edges were distributed at random.
The most popular random model is proposed by  \cite{Newman04}, where edges are reconnected randomly, under the constraint that the expected degree of each vertex corresponds to the degree of the vertex in the original graph.
 The last definition of cluster structure, and the most natural is related to  similarity between each pair of vertices, that includes local or global definitions of a cluster structure.  It is really natural to assume that cluster structures are groups of vertices that are close to each other.  Similarity measures are the foundations of traditional methods as detailed below. These include traditional distance measures such as Manhattan or Euclidean distances or computing correlations between rows of the adjacency matrix or random walk based similarities \cite{Pons05}. 

Once the definition of cluster structure is fixed, it is crucial to build efficient procedures and algorithms for the identification of such structures in the network. The ability to find and  to analyze such groups can provide an invaluable help in understanding and visualizing the  functional  components  of  the  whole  graph \cite{Girvan02,Newman04}. Classical techniques for data clustering, like hierarchical clustering, partitioning clustering and spectral clustering, detailed below,  are sometimes adopted for graph clustering too. Hierarchical clustering \cite{Hastie01} builds a hierarchy of nested clusters organized as a tree. partitioning clustering \cite{Pothen97} decomposes the graph  into a set of disjoint clusters. Given $N$ variables/nodes, it builds $k$ partitions of the data by satisfying:  (i) each group contains at least one point (ii) each point belongs to exactly one group.  In recent years, spectral clustering  has become one of the most widely used methods due to its speed and simplicity \cite{Luxburg07,Chung97,Ng02,Fortunato10}. This method extracts the geometry and local information of the dataset by computing  the top or bottom  eigenvectors of specially constructed matrices. The observations are  projected into this eigenspace to reduce the dimensionality of the problem and   $k$-means procedure is then applied in an easier subspace to detect clusters.


$k$-means, that belongs to  partitional clustering methods, aims to find a set of $k$ cluster centers of a dataset such that the sum squared of distances of each point to its closest cluster center is minimized. Lloyd’s 1957 procedure  \cite{Lloyd82} remains one of the widely used because of its speed and simplicity. It has been studied for several decades \cite{Lloyd82,Wu08} and many versions of this technique has recently been developed.  \cite{Xu19} proposed alternatives to Lloyd’s algorithm that preserves its simplicity, makes it more robust to initialization and relieves its tendency to get trapped by local minima. \cite{Lattanzi19} developed a new variant of $k$-means++ seeding algorithm \cite{Arthur07}  to achieve a constant approximation guarantee. 

\textbf{Our contribution.} Observed real networks differ from random graphs from their edge distribution and from their underlying structures. Erdös Renyi random graphs models \cite{Erdos59}, where all the pairs of nodes have equal probability of being connected by an edge, independently of all other pairs fail to model real observed graphs. Additionnally, stochastic block models are not always relevant to infer their structures. To remedy this problem, we developed a new random model, closely related to stochastic block model, but better suited to model graphs that have been  inferred from the observations. In practice, graphs that are studied are not known beforehand but often estimated.To achieve a good clustering recovery, random graph models are often associated to their similarity matrix to maintain the clustering structure of the graph. \cite{Wang16} developed a model to learn a doubly stochastic matrix which encodes the probability of each pair of data points to be connected,  used to normalize the affinity matrix such that the data graph is more suitable for clustering tasks.  \cite{Peng15} has shown that for a wide class of graphs, spectral clustering gives a good approximation of the optimal cluster.
 In our model, we assume that a group  does not emerge by chance but because there exists an underlying structure. This randomized version of the deterministic graph with exact cluster structure, is used to check whether it displays the original cluster structure. \cite{Sussman12,Lei15,Rohe11} proved consistency of spectral clustering applied to stochastic block models for some specific adjacency type matrices. Even if the consistency of spectral clustering has been proved for stochastic block models, there is no convergence guarantee for general models. Thus,  $k$-means can fail to reach the true underlying partitions of the graph. Moreover, spectral clustering technique fails to recover the original clusters when it comes to a higher randomization coefficient. This is mostly due to the computed eigenbasis that is not equally informative. 

In order to tackle this issue, we develop an alternative method to the spectral clustering that promote a sparse eigenvectors basis solution of an $\ell_0$ optimization problem,  corresponding to the indicator vectors of each cluster. Since the natural constrained $\ell_0$ is a NP-hard problem,  it was then replaced by its convex relaxation $\ell_1$ \cite{Ramirez13}. Actually we can show that the solution of the $\ell_0$ optimization problem is still the same when replacing the $\ell_0$-norm by the $\ell_1$-norm if we add a constraint on the maximum of the coefficients. Hence,  the algorithm turns out to solve an  $\ell_1$-penalty   optimization problem that is feasible and easy to implement, even for very large graphs. In a wider scope,  research papers have explored differently regularized  spectral clustering to robustly identify clusters in large networks.
Although \cite{Zhang18} and \cite{Joseph16} show the effect of regularization on  spectral clustering through graph conductance and respectively through stochastic block models. Equally, \cite{Lara19}, shows on a simple block model that the spectral regularization  separates the underlying blocks of the graph.


In this paper, we introduce  in Section \ref{section2} a new random graph model, used to solve spectral clustering (Section \ref{section3}) and its new variant  (Section \ref{section4}) objective function. We prove the  efficiency and accuracy of the variant algorithm  in Section \ref{section5} through experiments on simulated and  real  medical dataset (Section \ref{section6}).

\section{New random graph model}\label{section2}

\subsection{Notations}


This work considers the framework of an unweighted undirected graph  $G(V,E)$ with no self-loops consisting of  vertices $V=\left\{ 1,\dots,n \right\}$ and  $p$ edges connecting each pair of vertices.  An edge $e \in E$ that connects a node $i$ and a node $j$ is denoted by $e = (i,j)$. In this paper, we consider  that the existence of a link between two nodes in an interaction network is already inferred from the estimation of a statistical dependance measure. The graph $G$  is represented hereusing an adjacency matrix  $A = (A_{ij})_{(i,j) \in V^2}$ defined by 

\begin{center}
$A_{i,j} = \left\{
\begin{array}{ll}
1  \  \ \mbox{if there is an edge between } i \mbox{ and } j,\\
0 \  \  \mbox{otherwise.}
\end{array}
\right.$
\end{center}

Since the graph is undirected and with no self-loops, $A \in  \mathbb{M}_n(\mathbb{R})$ is a symmetric matrix with coefficients zero on the diagonal.

For each node $i$, the degree $d_i$ is defined as the number of edges incident to $i$ and is equal to : $d_i =\sum\limits_{j=1}^n A_{ij}$. We denote by $D$  the diagonal degree matrix  containing $(d_1,\dots ,d_n)$ on the diagonal and zero elsewhere. 

A subset $C \in V$ of a graph is said to be connected if any two vertices in $C$ are connected by a path in $C$. Non empty sets $ C_1, \dots, C_k $ form a partition of the graph   $G(V,E)$ if $C_i \cap C_j =\emptyset$ and $C_1\cup  \cdots \cup C_k=V$.   In addition, $C_i$ are called  connected components if there are no connections between vertices in $C_i$ and $\overline{C_i}$ for all i  in $ \left\{1, \dots, k \right\}$. 

We define the indicators of connected components $\textbf{1}_{C_i}$ whose entries are defined by:

\begin{center}
$(\textbf{1}_{C_i})_j = \left\{
\begin{array}{ll}
1  \  \ \mbox{if vertex } j \mbox{ belong to } C_i,\\
0 \  \  \mbox{otherwise.}
\end{array}
\right.$ 
\end{center}

\subsection{Graph models}

As mentioned in Section \ref{section1}, random graph models, in general, are not always relevant to represent the structure of a graph that has been  inferred from observations. To tackle this issue, we create a new random model with an underlying structure that is a randomized version of a deterministic graph with exact cluster structure.

\subsubsection{Ideal model}

We consider that the  graph $G_*(V,E)$ is the union of $k$ complete graphs  that are disconnected from each other.  We denote by $C_1, \dots, C_k$ the $k$ connected components of the graph, that match the $k$ clusters. We allow the number of vertices in each subgraph to be different. We denote by $c_1,\cdots,c_k $ $(\geq 2)$ their respective size ($\sum_{i=1}^k c_i=n$). To simplify, we assume that the nodes, labeled from $1$ to $n$, are ordered with respect to their block membership and in increasing order with respect to the size of the blocks. 

From a matricial point of view, the associated adjacency matrix $A_*$ is a $k$-block diagonal matrix of size $n$ of the form:

\[A_*=
\left[ 
\begin{array}{c@{}c@{}c}
  C_1  &  & \mathbf{0} \\
   &  \ddots & \\
  \mathbf{0} &  & C_k \\
\end{array}\right]
\]

where $C_1, \cdots,  C_k$ are symmetric matrices of size $c_1 \times c_1, \cdots, c_k \times c_k$.

\subsubsection{Perturbed model}

The reality is  that we consider the graph $G_*$ but we observe a randomized  version of this graph,  denoted by $\tilde{G}$.

We introduce the Erdös–Rényi model of a graph \cite{Erdos59,Stewart90},  one of the oldest and best studied random graph model. 

Given a set of $n$ vertices, we consider the variable $X_{ij}$ that  indicates the presence/absence of an edge between vertices $i$ and $j$. Then, for all $  \left\{ X_{ij} \right\} i.i.d.$, we have $ X_{ij} \sim B(p)$. 
Some edges have been added between the clusters and others have been removed within the clusters independently with respect to the same  probability $p$. The  adjacency matrix $B$ of the Erdos-Renyi graph of size $n$, whose upper entries are realizations of independent Bernoulli variables, can be written as
\begin{center}
\[
\left\{
\begin{array}{r c l}
B_{ij} &\sim& \mathcal{B}(p)  \  \ i.i.d, \ i<j \\
B_{ii}&=&0 \\
 B_{ij}&=&B{ji}
\end{array}
\right.
\]
\end{center}

The graph $\tilde{G}$ of the new model,  is derived from a deterministic graph with an exact cluster structure, whose edges have been disturbed Erdös Renyi random graph. For instance,  this perturbation may arise because of  a  partial  knowledge  of  the  graph.   

Let  $\tilde{A}$ be the adjacency matrix associated to $\tilde{G}$. $\tilde{A}$ is defined as follows:

 \begin{center}
$\tilde{A}$ $ =A_* \bigoplus\limits^2 B$
\end{center} where $\tilde{A}_{ij} =  A_{ij*} +  B_{ij} \  [2]$, namely, 

$\tilde{A}_{ij}= \left\{
\begin{array}{ll}
0  \  \ \mbox{ if } B_{ij}=1 \mbox{ and } A_{ij}=1 \\
 \mbox{ or } B_{ij}=0 \mbox{ and } A_{ij}=0,\\
1 \  \  \mbox{ if } B_{ij}=1 \mbox{ and } A_{ij}=0 \\
\mbox{ or }  B_{ij}=0 \mbox{ and } A_{ij}=1.
\end{array}
\right. $

\section{Graph clustering through spectral clustering
} \label{section3}

\subsection{Spectral clustering  algorithm}

Looking first at the ideal graph model, let $A_*$ and $D_*$ the adjacency and degree matrices associated to the graph $G_*$.

Spectral clustering algorithm is based on graph Laplacian matrices. Among them,  three different variants are used: \begin{itemize}
\item the Unormalized Laplacian:

\begin{center}
 $L_*=D_*-A_*$,
 \end{center}
\item the  Symmetric Laplacian: 

\begin{center}
$L_{sym*}=D_*^{-\frac{1}{2}}A_*D_*^{-\frac{1}{2}}$,
\end{center}

\item  the  Random Walk Laplacian:

\begin{center}
$L_{rw*}=I-D_*^{-1} A_*$. 
\end{center}
\end{itemize}
The original spectral clustering method has been proposed by  \cite{Luxburg07}  to cluster the nodes of the graph into $k$ connected components. The idea behind spectral clustering is to use the first $k$ eigenvectors (corresponding the $k$ smallest eigenvalues)  a normalized or unormalized version of the Laplacian matrix (derived from the adjacency one) to recover the connected components of the graph.  If these matrices are so appealing in graph clustering, it is because of the following proposition:

\begin{prop}\label{prop1}(Number of connected components):
The multiplicity $k$ of the eigenvalue $0$ of  $L_*$ and $L_{rw*}$ and the multiplicity of the generalized eigenvalue $0$ of $L_{sym*}$ are  equal the number of connected components $C_1 , \dots, C_k $ in the graph. For $L_*$ and $L_{rw*}$, the eigenspace associated to $0$ is spanned by the indicators of connected components $\left\{ 1_{C_i} \right\}_{1 \leq i \leq n}$. For $L_{sym*}$, the eigenspace associated to $0$ is spanned by $\left\{ D_*^{1/2} 1_{C_i} \right\}_{1 \leq i \leq n}$.
\end{prop}


We deduce from \textbf{Proposition 1} that  a particular basis of the associated eigenspace is  spanned by the connected components indicators. In addition,  the rows of the matrix resulting from the concatenation of the $k$ first  eigenvectors, are equal for indices corresponding to nodes in the same component.  Therefore, it is natural to  apply $k$-means to these rows to provide, by the same way the blocks. Moreover, as the graph is made of exactly $k$ connected components, the computation of the eigenvectors of $L_*$, $L_{sym*}$, $L_{rw*}$ enables to recover these components. 

\subsection{Limits}

Secondly, we consider  the perturbed version $\tilde{G}$ of the  graph $G_*$. Thus,  $\tilde{G}$  is no longer made of connected components, but of densely connected subgraphs that are sparsely connected to each other. These densely connected subgraphs represent somehow a perturbed version of the initial connected components that form the clusters. As our model is closely related to stochatic block models and if the perturbation is not too high, we can still hope that the rows of the $k$ concatenated eigenvectors  are still closed for indices corresponding to nodes in the same cluster.  But there is no theoritical guarantee that  it still contains enough information on the graph structure to detect these clusters using $k$-means procedure.


To overcome this issue, we developed an alternative to the standard spectral clustering, called $\ell_1$-spectral clustering, that aims at finding the $k$ underlying connected components of a graph $G_*$ with an exact cluster structure from its perturbed version $\tilde{G}$.

\section{$\ell_1$-spectral clustering, a new method for  connected component detection
} \label{section4}

To ensure a good recovery of the connected components,  the way  the eigenvectors basis is built    is of the highest importance.  The key is to replace the $k$-means procedure by the selection of  relevant eigenvectors that provide useful information about the structure of the data. 
$\ell_1$-spectral clustering focused on the graph adjacency matrix instead of the Laplacian matrix or its normalized version, and its good properties. The idea remains the same if we replace the adjacency matrix by the Laplacian or its normalized version as proved by \cite{Sussman12,Rohe11}.

We consider  the ideal  adjacency matrix $A_*$ associated to the graph $G_*(V,E)$, we assume in what follows that the eigenvalues  of $A_*$  are sorted  increasing order. And the same goes for the associated eigenvectors.

The  indicators ${\{\textbf{1}_{C_i}\}}_{n-k+1 \leq i\leq n}$ of the  connected components $C_{n-k+1}, \dots, C_{n}$  are the eigenvectors associated this time to the largest eigenvalues.  Theses eigenvalues   are equal to the degree coefficients of the connected components $d_{n-k+1}, \dots, d_n$. The $k$ first eigenvectors of $A_*$ (associated to the $k$ largest eigenvalues) are thus denoted  $v_{n-k+1}, \dots, v_n$. Let $V_{1,k}$ the matrix that contains $v_{n-k+1}, \cdots, v_n$ in columns and by $V_2,n$ the one that contains $v_{1}, \dots, v_{n-k}$. We define $\mathcal{V}_{1,k}^0 =\text{Span}\{v_{n-k+1},\dots,v_n\}$.

Unlike the traditional spectral clustering method,  $\ell_1$-spectral clustering does not directly use  the subspace spanned by the  eigenvectors associated to the largest eigenvalues to recover the connected components but computes another eigenbasis that promotes sparse solutions for the eigenvectors.

\subsection{General $\ell_0$ minimization problem}

Proposition 2 and 3 below show that the connected components indicators are solution of some specific problem.

 
 \begin{prop}\label{prop2}
The minimization problem ($\mathcal{P}_0$)
\begin{center}
$\underset{v\in \mathcal{V}_{1,k}^0 \backslash\{0\}}{\arg\min}  {\|v\|}_0 $
\end{center}
has a unique solution (up to a constant) given by $\textbf{1}_{C_{n-k+1}}$.
\end{prop}

In other words, $\textbf{1}_{C_{n-k+1}}$ is the sparsest non-zero eigenvector in the space spanned by the   eigenvectors associated to the $k$ largest eigenvectors.

\begin{proof}
We recall that $\|v\|_0=| \left\{ j : v_j \ne 0  \right\}|$. Let $v \in \mathcal{V}_{1,k}^0 \backslash \left\{ 0 \right\}$. Therefore, as $\textbf{1}_{C_{n-k+1}} \in \mathcal{V}_{1,k}^0$, $v$ can be decomposed as $v=\sum\limits_{j=n-k+1}^n \alpha_j \textbf{1}_{C_j}$ where $\alpha=(\alpha_{n-k+1},\dots ,\alpha_n)\in \mathbb{R}^k \  $ and $ \exists   j,  \  \alpha_j \ne 0$.

The connected components of sizes $c_{n-k+1},\dots, c_n$ are sorted in increasing order of size.
Therefore, by Proposition \ref{prop1}, $\|v\|_0=\textbf{1}_{\alpha_{n-k+1} \ne 0}c_{n-k+1}+\dots +\textbf{1}_{\alpha_n \ne 0} c_n$. The solution of ($\mathcal{P}_0$) is given by the vector in $\mathcal{V}_{1,k}\backslash\{0\}$ with the smallest $\ell_0$-norm  such that $\alpha=(\alpha_{n-k+1} , 0, \dots, 0)$ where $\alpha_{n-k+1} \ne 0$.
\end{proof}

We can generalize Proposition \ref{prop2} to find, iteratively and with sparsity constraint,  the other following indicators of  connected components. 

 For $i=n-k+2,\dots,n$, let $\mathcal{V}_{1,k}^i =\left\{v\in \mathcal{V}_{1,k} : v \perp \textbf{1}_{C_l}, l=n-k+1, \dots ,i-1 \right\}$.

\begin{prop}\label{prop3}
The minimization problem ($\mathcal{P}_i$)
\begin{center}
$\underset{v\in \mathcal{V}_{1,k}^i \backslash\{0\}}{\arg\min}  {\|v\|}_0 $
\end{center}
has a unique solution (up to a constant) given by $\textbf{1}_{C_i}$.
\end{prop}

Solving ($\mathcal{P}_0$) (Proposition \ref{prop2})  is  a  NP-hard problem.  In order to tackle this issue, we replace the ${\ell_0}$-norm  by its convex relaxation  ${\ell_1}$-norm. We can show that the solution of the $\ell_0$ optimization problem is still the same by replacing the $\ell_0$-norm by the $\ell_1$-norm, if we add the constraint on the maximum of the coefficients.

\subsection{General $\ell_1$ minimization problem to promote sparsity}\label{l1}

 In addition to the number of connected components, we assume that we know one representative of each component i.e. a node belonging to this component. This assumption is not so restrictive compared to traditional spectral clustering where the number of clusters is assumed to be known.  If we do not exactly know a representative for each component,  we can estimate them by first applying a rough partitioning algorithm or just an algorithm that aims to find  hubs of very densely connected parts of the graph.

Let $I_{n-k+1},\dots, I_n$ be the  row indices of the representative element of each component and let $\tilde{\mathcal{V}}_{1,k}^1 =\{v\in \mathcal{V}_{1,k}^0 : v_{I_j}=1\}$ for all  $j\in \left\{n-k+1, \dots,n   \right\}$. This is straightfoward to see that the  indicator vector of the smallest component is solution of the following optimization problem.

\begin{prop}\label{prop4}
The minimization problem  ($\mathcal{P}_1$)
\begin{center}
$\underset{v\in \tilde{\mathcal{V}}_{1,k}^1 }{\arg\min}  {\|v\|}_1 $
\end{center}
has a unique solution  given by $\textbf{1}_{C_{n-k+1}}$.
\end{prop}

\begin{proof}\label{proof2}
We recall that $\|v\|_1=\sum\limits_{i=1}^n |v_i|$. Let $v \in \tilde{\mathcal{V}}_k^1$. Therefore, as $\textbf{1}_{C_{n-k+1}} \in \mathcal{V}_{1,k}^0$, $v$ can be decomposed as $v=\sum\limits_{j=n-k+1}^n \alpha_j \textbf{1}_{C_j}$ where $\alpha=(\alpha_{n-k+1},\dots ,\alpha_n)\in \mathbb{R}^k \  $ and there exists  $ j \in \left\{n-k+1, \dots, n \right\} ,  \  \alpha_j \ne 0$.

The connected components of sizes $c_{n-k+1},\dots, c_n$ are sorted in increasing order of size.
Therefore, $\|v\|_1=\alpha_{n-k+1} c_{n-k+1}+\dots +\alpha_n c_n$. The solution of ($\mathcal{P}_1$) is given by the vector in $\tilde{\mathcal{V}}_{1,k}^1$ that satisfies $\|v\|_{+\infty}$ and  with the smallest $\ell_1$-norm  such that $\alpha=(\alpha_{n-k+1} , 0, \dots, 0)$ where $\alpha_{n-k+1}=1$.

\end{proof}

To simplify and without loss of generality, we assume that $I_{n-k+1}$ corresponds to the first index (up to a permutation).  We can rewrite ($\mathcal{P}_1$) (Proposition \ref{prop4}) as:

\begin{center}
$\underset{(1,v) \in \mathcal{V}_{1,k}}{\underset{v \in \mathbb{R}^{n-1}}{\arg\min}} {\|v\|}_1 $
\end{center}

Constraints in ($\mathcal{P}_1$) can be moved to the following  equality contraints:

\begin{prop}\label{prop5}
Let $w$ be the first column of $V_{2,n}^T$. We define $W$ as the matrix $V_{2,n}^T$ whose first column $w$ has been deleted.

The  minimization problem ($\tilde{\mathcal{P}}_1$) 
 \begin{center}
$\underset{W v=-w}{\underset{v \in \mathbb{R}^{n-1}}{\arg\min}} {\|v\|}_1 $
\end{center} 
with $w=V_{2,n}^{(1)}$ and $W=V_{2,n}^{(n-1)}$ has a unique  solution $\tilde{v}$ equals to $\textbf{1}_{C_{n-k+1}}$ such that  $\tilde{v}=(1,v)$.
\end{prop}

\begin{proof}
We recall that $V_{2,n}$ is the restriction of the eigenvectors matrix to the $n-k$ first  columns. Because the columns of this matrix  form an orthogonal basis, $v\in \mathcal{V}_{1,k}$ is equivalent to $V_{2,n}^T v=0$.
Thus, $\tilde{v}=(1,v)$ satisfies the equation: $V_{2,n}^{(1)} + V_{2,n}^{(n-k-1)} v$, where $V_{2,n}^{(1)}$ is the first row of $V_{2,n}$ and $V_{2,n}^{(n-k-1)}$ the matrix  $V_{2,n}$  whose first row has been deleted.

For all $ \tilde{v}=(1,v), $
 \begin{align*}
 V_{2,n}^T\tilde{v} & =V_{2,n}^{(1)} +  V_{2,n}^{(n-k-1)} v\\
 &=0 \\
  \Leftrightarrow \qquad V_{2,n}^{(n-k-1)} v &= -V_{2,n}^{(1)}
\end{align*}

Note that in Proposition \ref{prop5}, $V_{2,n}^{(1)}$ and  $V_{2,n}^{(n-k-1)}$ are denoted $w$ and $W$.

\end{proof}

%

\textbf{Remark:}  Constraint problem ($\tilde{\mathcal{P}}_1$) (Proposition \ref{prop5}) can be equivalently written as the following  penalized problem:  

\begin{center}
$\underset{v \in \mathbb{R}^{n-1}}{\arg\min}  \ {\|W v+w\|}_2^2 + \lambda {\| v \|}_1$.
\end{center} where $\lambda >0$ is the regularization parameter that controls the balance between the constraint and the sparsity norm, $W \in \mathbb{M}_{n-k,n-1}$ is the matrix $V_{2,n}^T$ whose first column $w$ has been deleted.

In the following, we will provide an algorithm based on the contraint problem ($\tilde{\mathcal{P}}_1$) introduced in Proposition \ref{prop5}.


\section{$\ell_1$-spectral clustering algorithm}\label{section5}

Now, we only consider graphs with  an exact cluster structure whose edges have been perturbed by a coefficient $p \in [0,1]$.

$\ell_1$-spectral clustering algorithm  is developed in a Matlab software. 
Starting with the number of blocks $k$ of an adjacency matrix $A$ and the column index of one representative element of each block $I_{n-k+1}, \dots, I_n$,  the pseudo-code for the algorithm is presented in  Algorithm \ref{algo}.

Steps from  $3$  to  $14$ are dedicated to the recovery of the indicators of connected components.
The minimization problem introduced in Section \ref{l1}  is solved using the $\ell_1$-eq function of the Matlab optimization package $\ell_1$-magic \cite{Candes05}. Vector $\tilde{v}_j$ contains the solution of the minimization problem (step 11).

To find the other connected components indicators, we add the constraint of being orthogonal to the previous computed vectors by deflating the matrix $A$ (step 13) and we do the same to estimate the other connected component indicators.

 Let $F$ be the concatenation of the vectors $\tilde{v}_j$. As the algorithm is applied on a perturbed adjacency matrix, the elements in $F$ are not exactly equal to one or zero but are very close to one for the indices associated to edges belonging to a same cluster and to zero for the remaining ones. Therefore, we shrink the solution (steps 16 to 20):
 
  For all $j=1,\dots,n$, for all $i=1,\dots k$, 
  
  $F_{ij}=\left\{
\begin{array}{ll}
1  \  \ \mbox{if}  \ F_{ij} >\frac{1}{2}, \\
0 \  \  \mbox{if}  \ F_{ij} \leq \frac{1}{2}.
\end{array}
\right.$

The indicators of the clusters are given by the $k$ column vectors of  $F$.

\begin{algorithm}[tb]
   \caption{$\ell_1$-spectral clustering algorithm}
   \label{algo}
\begin{algorithmic}[1]
   \STATE {\bfseries Input:} number of clusters  $k$ ,  adjacency matrix $A$,   indices of representative elements $Index=[I_{n-k+1},\dots, I_n]$. 
   \STATE Initialize $F=[]$.
   \STATE \COMMENT{Recovery of  the indicators of  the connected components}
   \FOR{$j=1$ {\bfseries to} $k$}
   \STATE Eigen decomposition $[V,U]$ of $A$: $A=VU^tV$.
      \STATE Sort in ascending order the eigenvalues and the associated eigenvectors of $A$.
        \STATE Form the matrix $V_{2,n}$ by stacking the $n-k+j-1$ eigenvectors associated to the smallest eigenvalues. 
        \STATE Computation of constraints of  the $l_1$ minimization problem
        \STATE Compute $T=V_{2,n}^t$
        \STATE Compute $W=T^{-Index[j]}$ and $w=T^{Index[j]}$(where $T^{Index[j]}$ is the column $Index[j]$ of $T$ and $V^{-I_j}$ the matrix $T$ wihtout the column  $Index[j]$)
        \STATE Compute  the solution $v^j$ of the following problem 
        \begin{center}
$\underset{Wv=-w}{\underset{v \in \mathbb{R}^{n-1}}{\arg\min}} {\|v\|}_1 $.
\end{center} 
\STATE Recovery of the $j_{th}$ cluster: 

 $\tilde{v}_j=[v^j_{1}  \ v^j_{2} \  \dots \ v^j_{Index[j] -1} \ 1 \ v^j_{Index[j]+1} \  \dots \  v^j_{n}] $
\STATE $F$ concatenation of the $j_{th}$ clusters and deflation of $A=A - \tilde{v}_j   ^t \tilde{v}_j$ to recover the other indicator vectors
\ENDFOR
\STATE $F((F>0.5))=1$
 \STATE $F((F \leq 0.5))=0$
   \STATE {\bfseries Output:} $k$ column vectors $F$.
\end{algorithmic}
\end{algorithm}

\section{ $\ell_1$-spectral clustering applications}\label{section6}

\subsection{Spectral clustering and $\ell_1$-spectral clustering on simulated dataset}

\subsubsection{Performances}

In Section \ref{section5}, we  introduced a new algorithm (called $\ell_1$-spectral clustering) that aims to detect cluster structures in complex graphs.    To  illustrate  graphically the  performances  of  this method,  we simulated a perturbed version of a graph with an exact group structure. The associated  adjacency matrix is composed of  $k \in [5,10]$ blocks of size $c_{n-k+1},\dots, c_n \in [10,20]$. Let $p$ be the level of Bernoulli noise applied on the adjacency matrix. 

 Once  the  matrix  is  disturbed by a strictly positive  coefficient,  we no longer  have  exact  block structures.  To recover it, we applied the traditional spectral clustering algorithm  and the new $\ell_1$-spectral clustering algorithm.  Figure \ref{recovery} gives the performances of both algorithm with a perturbation coefficient of $p=2$. 
 
\begin{figure}[ht]
\vskip 0.2in
\begin{center}
\centerline{\includegraphics[width=\columnwidth]{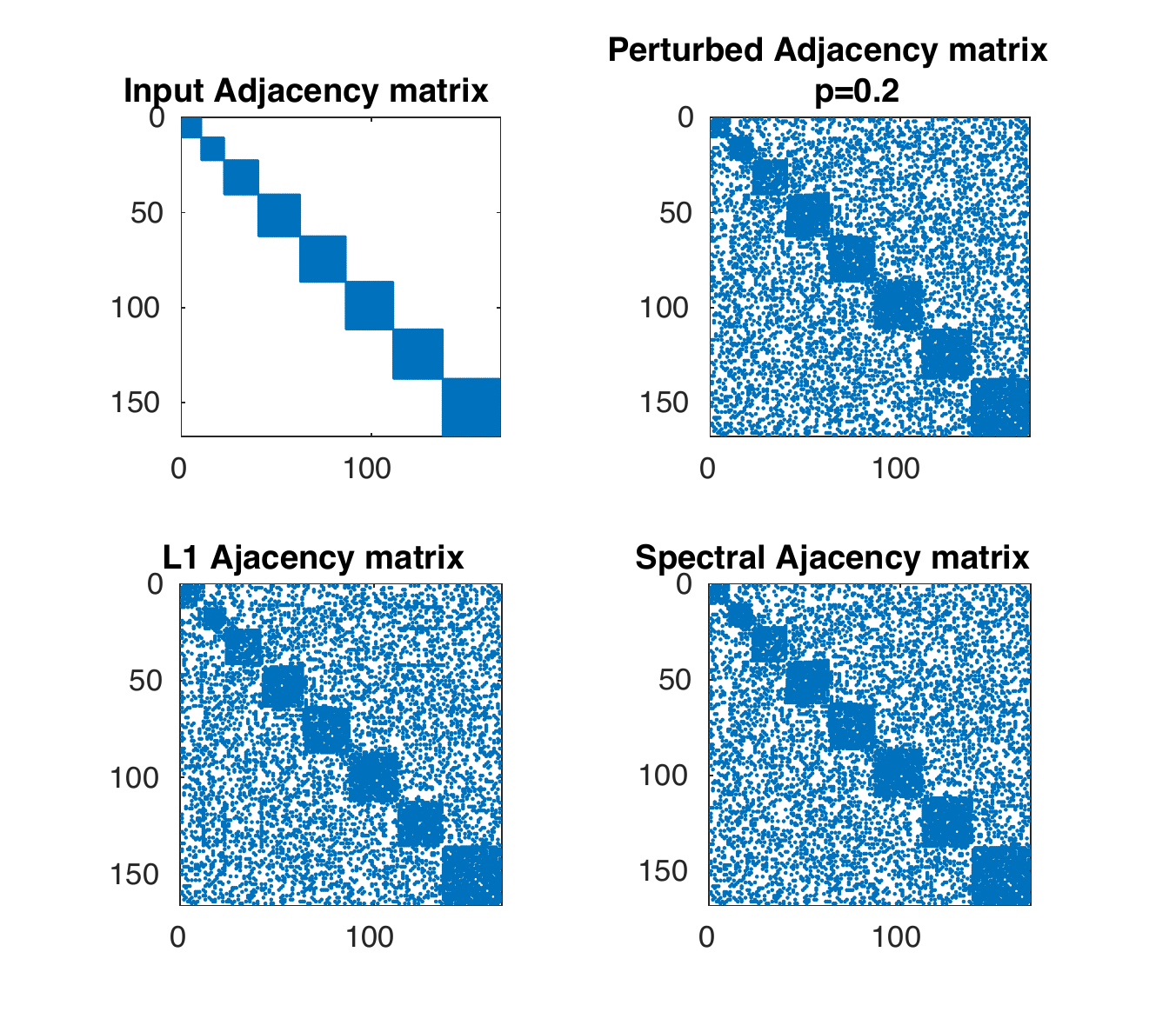}}
\caption{Input Adjacency matrix: Adjacency matrix with exact community structure. Perturbed Adjacency matrix $p=0.2,0.3$: Adjacency matrix after perturbation. L1 Adjacency matrix: Adjacency matrix recovery after $\ell_1$-spectral clustering application. Spectral Adjacency matrix: Adjacency matrix recovery after spectral clustering application.}
\label{recovery}
\end{center}
\vskip -0.2in
\end{figure}

We can notice that our model performs well in this task as both methods  effectively recovers the clustering structure, which indicates the robustness of our model.

\subsubsection{Robustness to perturbations}

Then, we tested the robustness under perturbations of the spectral clustering ang $l_1$-spectral clustering algorithms. Let $p$ be the level of Bernoulli noise, discretized in this section between $0$ and $0.4$. In this experiment, we simulate $100$  graphs with $k \in [5,10]$ clusters of size $c_{n-k+1},\dots, c_n \in [10,20]$. 

We introduce the block membership function: for all node $i  \in \left\{1,\dots, n \right\}$ of a graph $G(V,E)$ made of block structures of size $c_{n-k+1},\dots, c_n$, \begin{align*}
  \tau \colon & V \to \left\{n-k+1,\dots, n \right\}\\
  &i \mapsto c.
\end{align*} 

For each value of $p$, we test the perfomances of both algorithms to recover the clusters of the graphs.   The performances of the algorithms were evaluated by computing the percentage of missassigned nodes in average defined as $\frac{1}{100} \sum\limits_{j=1}^{100} |\left\{i\in V : \tau(i) \ne \hat{\tau}_j(i)\right\}|$, where $\tau_j$ is the block membership function and $\hat{\tau}_j$ is the estimated membership function for the $j$-th model. The results are plotted in Figure \ref{comparaison}.

\begin{figure}[ht]
\vskip 0.2in
\begin{center}
\centerline{\includegraphics[width=\columnwidth]{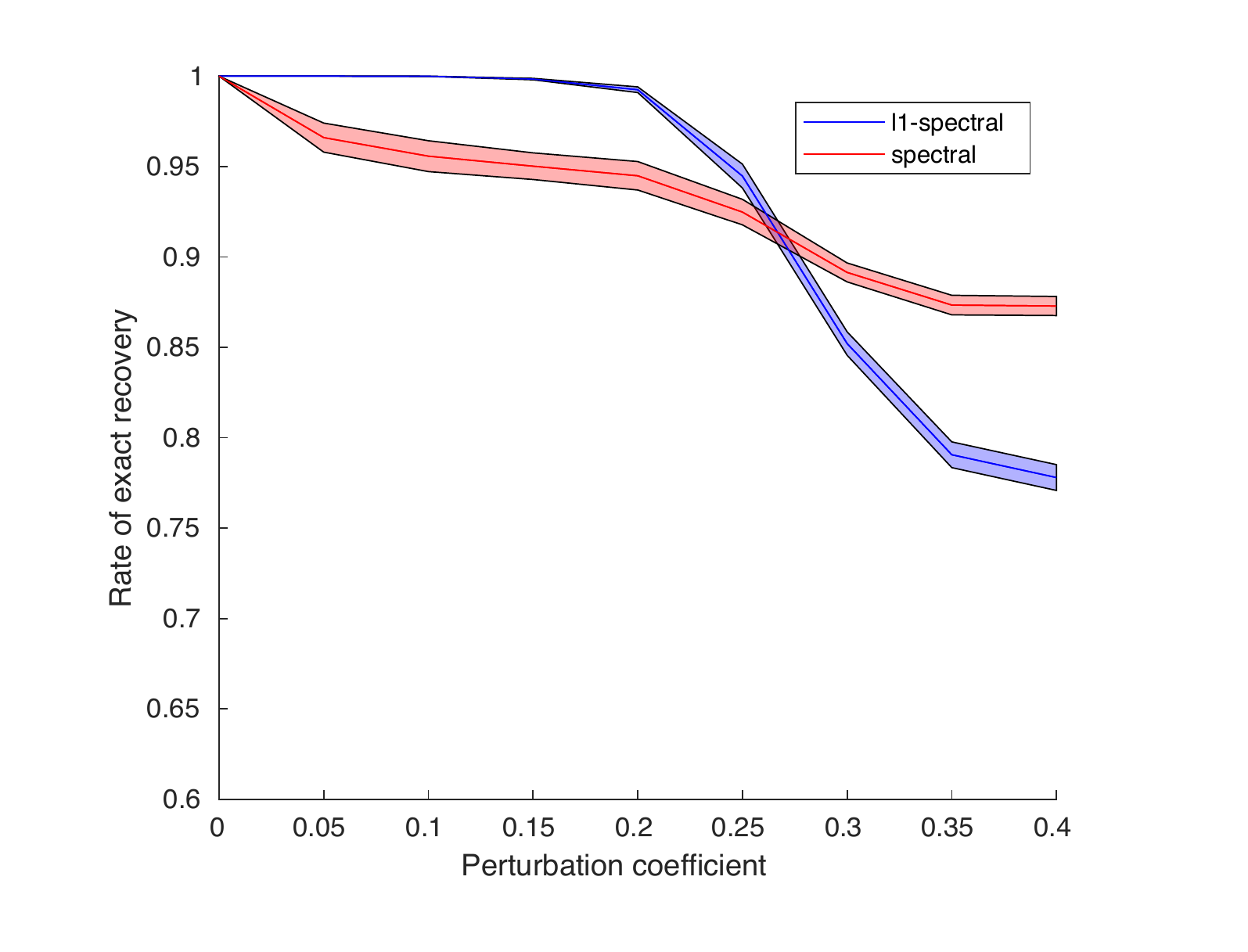}}
\caption{Fraction of nodes correctly classified using spectral (red line) and $\ell_1$-spectral clustering (blue line) under increasing perturbation coefficient. }
\label{comparaison}
\end{center}
\vskip -0.2in
\end{figure}

Figure \ref{comparaison} captures the fraction of nodes correctly classified and the associated region of confidence when $\ell_1$-spectral clustering (blue) and spectral clustering (red) are applied under increasing perturbation coefficient.

\subsubsection{Results}

Both simulations show that the perturbation coefficient has an impact on the performance of spectral clustering and $\ell_1$-spectral clustering. Moreover, we observe that $\ell_1$-spectral clustering works better  on  simulated  data  for small perturbations (up to $30\%$ Bernoulli noise) than spectral clustering. Thus, the new technique provides powerful results on small perturbations (rate of exact assignment is equal or very close to one).

\subsection{$\ell_1$-spectral clustering on real dataset}

\subsubsection{FLORINASH dataset}

This section is dedicated to experimental studies to assess the performances of our method through real dataset. Experiments have been performed on R using the packages igraph, PLNmodels \cite{Chiquet19}. The dataset we used belongs to the project FLORINASH that proposes an innovative research concept to address the role of intestinal microfloral activity in Non-Alcoholic Fatty Liver Disease (NAFLD).

Hepatic steatosis is  often observed in obese patients and is a preliminary stage  to non-alcoholic fatty liver disease.  The studied cohort \cite{Hoyles18} is made of   obese patients featured with hepatic steatosis. It  has been deeply studied and numerous clinical and biochemical data sets are available. We ran an ancillary study on $51$ control and $6$ diabetic patients with a median age of $42,50$ years, and characterized by a median body weight of $124,125$ kg  and a glycemia of $5.8,6.5$ mM.

The underlying dataset includes the output of  sequencing 16S rDNA gene from liver biopsies to study microbial composition and diversity of obese patients. The standard approach to analyzing 16S rDNA sequence data  relies on clustering reads by sequence similarity into Operational Taxonomic Units (OTUs). All OTUs are assigned to a taxonomic rank (phylum, class, order, family, genus and species). The standard way of representing the community structure inferred from microbial data is by means of an abundance table, where the rows correspond to samples ($57$ patients) and columns to features ($831$ microbial taxa).  The goal of this analysis is to detect clusters of OTUs at their family taxonomy level according to their abundance by patients ($53$ OTUs at this taxonomy).  Our aim is to identify the associations between the different microbial families by reconstructing the ecological network and make a direct comparisons between the two groups of patients.

\subsubsection{Results}

Microbiome data is compositional because the information that abundance tables contain is relative, the total number of counts per sample being highly variable. Few universal multivariate models are available for compositional  data and existing models often impose undesired constraints on the dependency structure. To tackle this issue, we use the Poisson Log Normal model \cite{Chiquet19}. We use the framework of graphical models  to model the dependency structure of the dataset. 

From the graph modeled, we deduce the adjacency matrix and the score of each underlying  hubs \cite{Kleinberg98}. A hub is a node with a number of links that greatly exceeds the average, also called a high degree node. A node is given a high hub score by linking to a large number of nodes.The number of hubs selected give us the total number of clusters do be detected. $l_1$-spectral clustering applied on the adjacency matrix  outputs $4$ (respectively $5$) clusters in control and diabetic patients (Figure \ref{comparaison}).

\begin{figure}[ht]
\vskip 0.2in
\begin{center}
\centerline{\includegraphics[width=\columnwidth]{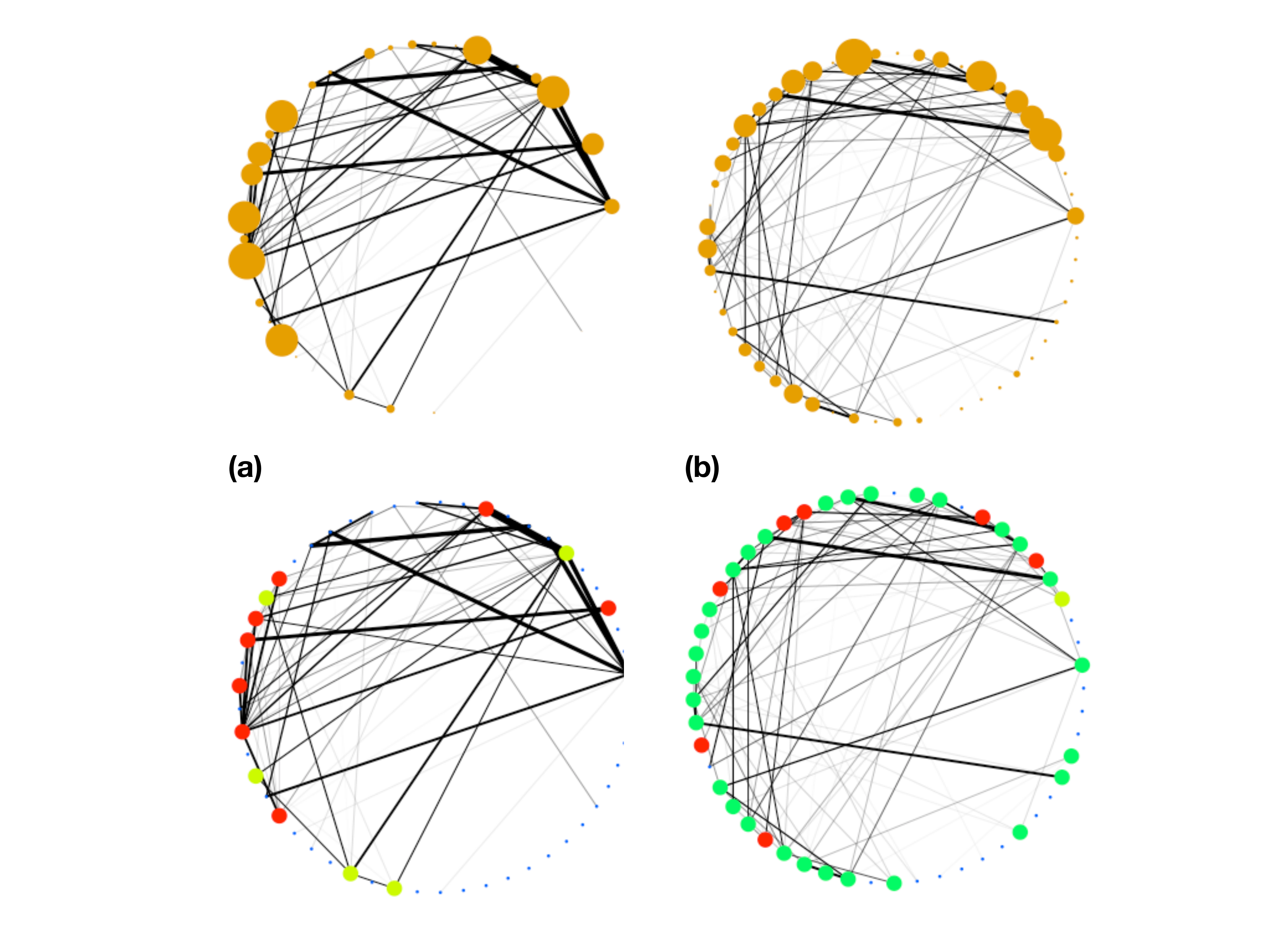}}
\caption{Application of $\ell_1$-spectral clustering on the cohort  composed of patients with and without diabetes. (a) Graph representing hubs and clusters related to  diabetic patients (b) Graph representing hubs and clusters related to healthy patients.}
\label{comparaison}
\end{center}
\vskip -0.2in
\end{figure}

\begin{figure}[ht]
\vskip 0.2in
\begin{center}
\centerline{\includegraphics[width=\columnwidth]{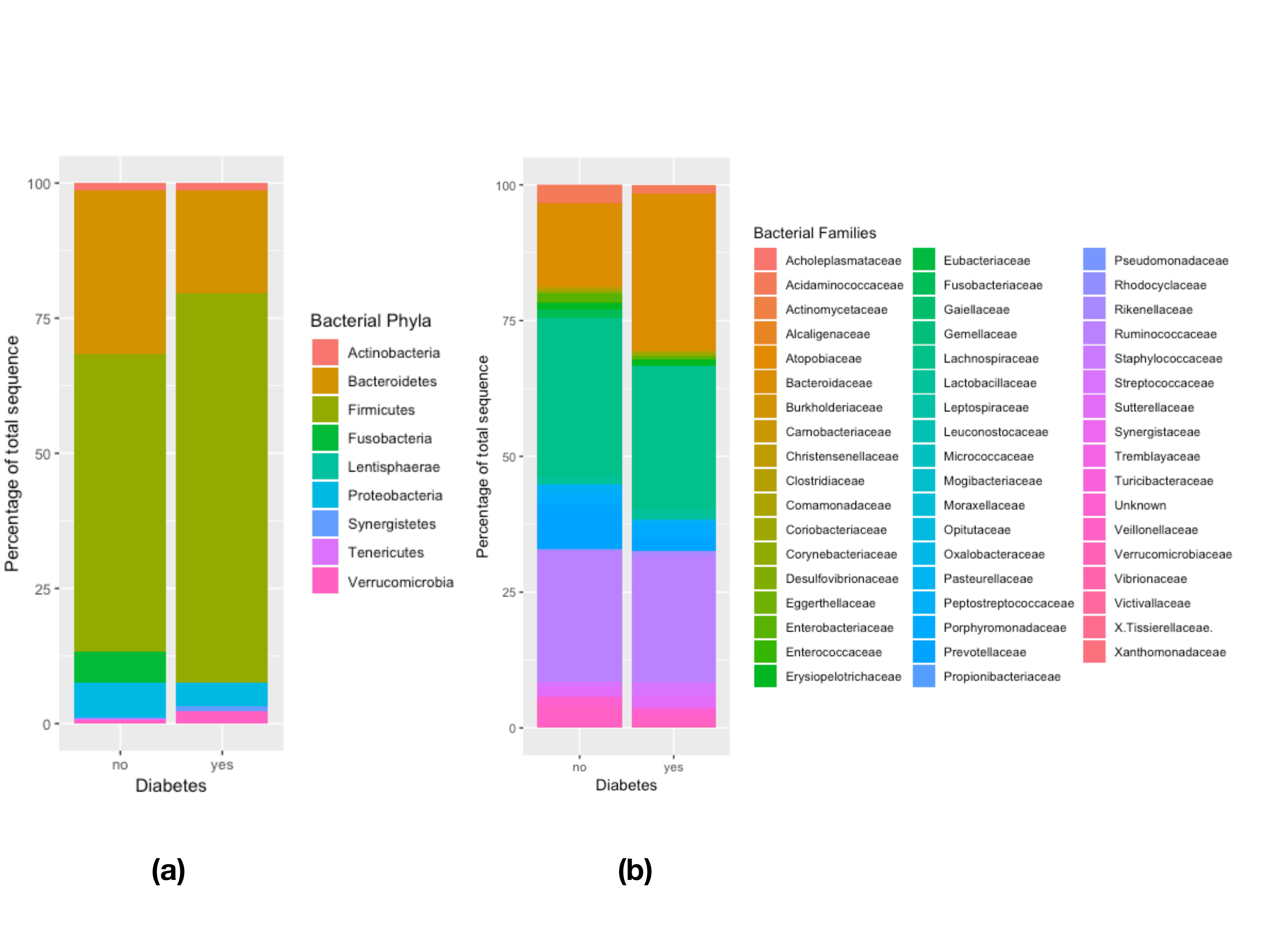}}
\caption{(a) Proportion of Phyla in control and diabetic patients (b) Proportion of Families in control and diabetic patients.}
\label{phyla}
\end{center}
\vskip -0.2in
\end{figure}

Figure \ref{phyla}   shows  that the gut microbiota from control and diabetic patients is characterized by two dominating phyla Firmicutes and Bacteroidetes (\cite{Ley05,Gill06}). We also added to the knowledge that in diabetic patients there was a disappearance of the frequency of the variable Verrucomicrobiaceae. This is also in agreement with the data from literature which demonstrate that this bacteria could be considered as a probiotic controlling metabolic diseases \cite{Everard13}. Eventually, from $\ell_1$-spectral clustering, we identified that a novel variable i.e. the Fusobacteriaceae are important discriminant signatures of the non diabetic group while that of the diabetic group is related to the Firmicutes variable. Altogether, our algorithm was validated by the previously published findings from the FLORINASH cohort and even add to the knowledge that some specific variables could be  associated with the diabetic or non diabetic signatures. 

\section{Conclusion}

We present $\ell_1$-spectral clustering,  a novel variation of spectral clustering algorithm based on promoting a sparse eigenvectors basis that provides information about the structure of the system observed. The associated graph is assumed to contain  connected subnetworks. We characterized the indicators of these subnetworks  as the ones that have the  minimal $\ell_1$-norm  with  respect  to  a  specific  restricted  space.   $\ell_1$-spectral clustering benefits from this feasible objective function as a substitution of the $k$-means step of the traditional spectral clustering. Its effectiveness and robustness to small noise perturbations  is confirmed by simulated and real examples.

\newpage 

\bibliography{example_paper}
\bibliographystyle{plain}

%
%
%

\end{document}